\documentclass{article}

\newcommand{\BiblioStyle}{
	\bibliographystyle{alpha}
}

\usepackage[utf8x]{inputenc}
\usepackage{ucs}

\usepackage{amsmath}
\usepackage{amsfonts}
\usepackage{amssymb}
\usepackage{amsthm}
\usepackage{stmaryrd} 
\usepackage{bbold} 

\IfFileExists{subcaption.sty}{
	\usepackage{subcaption}
	\captionsetup{compatibility=false}
}{
	\usepackage{caption}
	\DeclareCaptionSubType{figure}
	\makeatletter
	\caption@For{subtypelist}{%
		\newenvironment{sub#1}%
		{\caption@withoptargs\subcaption@minipage}%
		{\endminipage}}%
	\newcommand*\subcaption@minipage[2]{%
		\minipage#1{#2}%
		\setcaptionsubtype\relax}
	\makeatother
}

\usepackage{url}
\usepackage{afterpage}

\usepackage{multirow}
\usepackage{colortbl}
\usepackage{tablefootnote}

\usepackage{xspace}

\usepackage{tikz}
\usetikzlibrary{calc,positioning}

\usepackage{textgreek}

\newcommand{\PrintBiblio}{
	\BiblioStyle
	\bibliography{/home/leshyk/PhD/library.bib}
}
\newtheorem{proposition}{Proposition}
\newtheorem{example}{Example}

\newcommand{\A}{\mathcal{I}}

\newcommand{\PS}{\mathbb{D}} 
\newcommand{\K}{\mathbb{D}}

\newcommand{\stab}{\mathtt{Stab}}
\newcommand{\rbst}{\mathtt{Rbst}}

\newcommand{\set}[1]{\{#1\}}

\newcommand{\M}{\mathcal{M}}

\newcommand{\AlgoName}{\textSigma\textomikron\textphi\textiota\textalpha\xspace}
\newcommand{\ThetaAlgoName}{$\theta$-\textSigma\textomikron\textphi\textiota\textalpha\xspace}

\newcommand{\tikznamedpicture}[3][0]{
	\newcommand{#2}[#1]{ \begin{tikzpicture}#3\end{tikzpicture} }
}

\newcommand{\concept}[2]{$\left(#1;#2\right)$}

\tikznamedpicture{\putStabLattice}{[
	every node/.style={draw,rectangle,font={\tiny}},
	node distance= 0.5cm and 0.3cm
	]
	\node(g1){\concept{1}{\set{i_1,i_3}}};
	\node(g2)[right=of g1]{\concept{2}{\set{i_2,i_3}}};
	\node(g3)[right=of g2]{\concept{3}{\set{i_3,i_4}}};
	\node(g4)[right=of g3]{\concept{4}{\set{i_3,i_5}}};
	\node(g5)[right=of g4]{\concept{5}{\set{i_6}}};
	\node(bottom)[below=of g3] {\concept{\emptyset}{\A}}
	edge(g1)
	edge(g2)
	edge(g3)
	edge(g4)
	edge(g5);
	\node(g1234)[above=of $(g2)!0.5!(g3)$,very thick] {\concept{\bf 1234}{\bf\set{i_3}}}
	edge(g1)
	edge(g2)
	edge(g3)
	edge(g4);
	\node(top)[above=of g1234-|g3]{\concept{12345}{\emptyset}}
	edge(g1234)
	edge(g5);
}

\usepackage[lined,boxed,commentsnumbered,linesnumbered]{algorithm2e}
\SetAlCapSkip{10pt}

\begin{document}

\title{
  Mining Best Closed Itemsets \\ for Projection-antimonotonic Constraints \\ in Polynomial Time
}

\newcommand{\inst}[1]{$^#1$}
\author{
	Aleksey Buzmakov\inst{1} \and Sergei O. Kuznetsov\inst{1} \and Amedeo Napoli\inst{2}\\
	\inst{1}National Research University Higher School of Economics \\ Moscow / Perm, Russia\\
	\inst{2} LORIA (CNRS -- Inria -- University of Lorraine) \\ Vandœuvre-lès-Nancy, France\\
	avbuzmakov@hse.ru, skuznetsov@hse.ru, amedeo.napoli@loria.fr
}

\maketitle

\begin{abstract}
  The exponential explosion of the set of patterns is one of the main challenges in pattern mining. This chalenge is approached by introducing a constraint for pattern selection. One of the first constraints proposed in pattern mining is support (frequency) of a pattern in a dataset. Frequency is an anti-monotonic function, i.e., given an infrequent pattern, all its superpatterns are not frequent. However, many other constraints for pattern selection are neither monotonic nor anti-monotonic, which makes it difficult to generate patterns satisfying these constraints.
  In order to deal with nonmonotonic constraints we introduce the notion of ``projection antimonotonicity'' and \AlgoName{} algorithm that allow generating best patterns for a class of nonmonotonic constraints. Cosine interest, robustness, stability of closed itemsets, and the associated \textDelta-measure are among these constraints. \AlgoName starts from light descriptions of transactions in dataset (a small set of items in the case of itemset description) and then iteratively adds more information to these descriptions (more items with indication of tidsets they describe).
  In the experiments, we compute best itemsets w.r.t. some measures and show the advantage of our approach over postpruning approaches.
\end{abstract}

\section{Introduction}
Interestingness measures were proposed to overcome the problem of combinatorial explosion of the number of valid patterns that can be discovered in a dataset~\cite{Vreeken2014}.
For example, pattern support, i.e., the number of transactions covered by the pattern, is one of the most famous measures of pattern quality.
In particular, support satisfies the property of anti-monotonicity (aka ``a priori principle''), i.e., the larger the pattern is the smaller the support is~\cite{Mannila1994,Agrawal1994}.
Many other measures can be mentioned such as
pattern stability~\cite{KuznetsovStability2007,Roth2008},
margin closeness~\cite{Moerchen2011},
MCCS~\cite{Spyropoulou2013},
cosine interest~\cite{Cao2014},
pattern robustness~\cite{Tatti2014},
etc.

Some of these measures (e.g., support, robustness for generators~\cite{Tatti2014}, or upper bound constraint of MCCS~\cite{Spyropoulou2013}) are ``globally anti-monotonic'', i.e., for any two patterns $X \sqsubseteq Y$ (where $\sqsubseteq$ stays for containment or subsumption relation in the pattern language) we have $\M(X) \geq \M(Y)$, where $\M$ is a measure.
When a measure is anti-monotonic, it is relatively easy to find patterns whose measure is higher than a certain threshold (e.g. patterns with a support higher than a threshold).
In contrast some other measures are called ``locally anti-monotonic'', i.e., for any pattern $X$ there is an immediate subpattern $Y \prec X$ such that $\M(Y) \geq \M(X)$.
The corresponding constraint induces an accessible system~\cite{Boley2010} in itemset data. Indeed, for any itemset selected  by a locally anti-monotonic constraint, one can always find a smaller selected itemset different only in one item. The good strategy in this case is extension of a pattern $Y$   only to patterns $X$ such that $\M(X) \leq \M(Y)$. For example, cosine interest~\cite{Cao2014} is ``locally anti-monotonic,'' some other examples can be found in~\cite{Boley2010}.

The most difficult case is when a nonmonotonic measure is not even locally anti-monotonic.
The valid patterns can be selected by postpruning, i.e., by finding a (large) set of patterns satisfying an anti-monotonic constraint and pruning them w.r.t. the chosen nonmonotonic measure~\cite{Roth2008,Moerchen2011,Tatti2014}. For that one can rely on certain heuristics such as the one used in leap search~\cite{Yan2008}. More elaborated approaches allow constructing a measure from (anti-)monotonic primitives~\cite{Soulet2005,Cerf2008}. These approaches find a good anti-monotonic relaxation of the measure for the dataset in hand. Another interesting approach for dealing with non-monotonic constraints is search for a closure operator on the set of patterns adequate for the constraint in question \cite{Soulet2008}. 

In this paper we deal with a recently introduced algorithm \AlgoName{}, i.e. Sofia, for ``Searching for Optimal Formal Intents Algorithm''. \AlgoName was applied for an interval-tuple data \cite{BuzmakovPKDD2015}.
In this paper we apply \AlgoName for extracting the best itemsets w.r.t. a wide class of constraints. We introduce the polynomial version of the algorithm by accordingly adjusting the threshold and deeply studying the properties of the involved measures.
Our algorithm is applicable to a class of measures called ``projection-antimonotonic measures'' or more precisely ``measures anti-monotonic w.r.t. a chain of projections''.
This class includes globally anti-monotonic measures such as support, locally anti-monotonic measures such as cosine interest and some of the nonmonotonic measures such as stability or robustness of closed patterns.
We should notice that this class of measures is not covered by the previously introduced approaches. In particular, for the primitive-based approaches \cite{Soulet2005,Cerf2008} it is not clear how one can express certain measures from our class, e.g., stability and robustness, by means of the primitives. On the other hand, the approach for finding adequate closure \cite{Soulet2008} could be applied for stability and robustness, but the number of classes of equivalences that should be enumerated is likely to be high and accordingly the efficiency of the approach is likely to be low. Furthermore, neither of these approaches ensure a polynomial complexity of the algorithm.

In the experimental part of the paper we show that \AlgoName can be efficiently used to mine itemesets w.r.t. a constraint based on \textDelta-measure, a polynomially computable analog of stability and robustness. It significantly outperforms the postpruning approaches based on best known algorithms for mining closed itemsets. We should mention that comparison of \AlgoName with primitive-based approaches or with the approach for finding adequate closure is not possible since it requires a heavy study of efficiently expressing stability and robustness in terms of the primitives.

In the rest of the paper we work with itemsets and accordingly we use the word 'itemset' instead of 'pattern'. 
The remainder of the paper is organized as follows. Since the lattice of closed itemsets (concept lattice) is of high importance for concise representation of itemsets~\cite{Pasquier1999}, we use
the language of Formal Concept Analysis (FCA)~\cite{Ganter1999} and pattern structures~\cite{Ganter2001} which are introduced in Section~\ref{sect:fca}.
Then, \AlgoName{} algorithm is detailed in Section~\ref{sect:algo} for projection-antimonotonic measures. In the next section we discuss cosine interest, robustness, and stability that are examples of such measures.
Experiments and a discussion on \AlgoName efficiency are proposed in Section~\ref{sect:experiment}, before the conclusion.

\section{Preliminaries}\label{sect:fca}
\subsection{Binary Dataset}\label{sect:fca-pattern-structures}

\begin{figure}
  \centering
  \begin{subfigure}{\columnwidth}
    \centering
    \resizebox{0.5\columnwidth}{!}{
    \begin{tabular}{l|cccccc}
      \hline\hline
      & $i_1$ & $i_2$ & $i_3$ & $i_4$ & $i_5$ & $i_6$ \\
      \hline
      $t_1$ & x && x &&&\\
      $t_2$ && x & x &&&\\
      $t_3$ &&& x & x &&\\
      $t_4$ &&& x && x &\\
      $t_5$ &&&   &&& x \\
      \hline\hline 
    \end{tabular}
    }
    \caption{A binary dataset.}
    \label{tbl:stability-context}
    \vspace{3mm}
  \end{subfigure}
  \begin{subfigure}{\columnwidth}
    \centering
    \resizebox{\columnwidth}{!}{
      \putStabLattice
    }
    \caption{
      A concept lattice. Concept extents are given by their indices, i.e., $\{t_1,t_2\}$ is given by $12$.
    }
    \label{fig:stability}
  \end{subfigure}
  \caption{A binary dataset and the corresponding concept lattice.}
  \label{fig:stab-example}
	\vspace*{-7mm}
\end{figure}

FCA is a very convenient formalism for describing models of itemset mining and knowledge discovery~\cite{Ganter1999}. Since~\cite{Pasquier1999} lattices of closed itemsets (concept lattices) and closed descriptions are used for concise representation of association rules.
FCA gives a formalism for itemset mining. For more complex data such as sequences, graphs, interval tuples, and logical formulas one can use an extension of the basic model, called  pattern structures~\cite{Ganter2001}. With pattern structures one defines closed descriptions that give a concise representation of association rules for different types of descriptions with a partial order of ``part-whole"  (e.g., subgraph isomorphism order) or ``is a'' (e.g., ``class-subclass") giving rise to a semilattical  similarity operation $\sqcap$~\cite{Kuznetsov2005,KaytoueIJCAI2011}.



A \textit{binary dataset} is a triple $\K=(T,\A,R)$, where $T$ is a set of transaction identifiers, $\A$ is a set of items and $R \subseteq T \times \A$ is incidence relation giving information about items related to every transaction. 
A pattern structure or a (general) dataset is a triple $(T,(D,\sqcap),\delta)$, where $(D,\sqcap)$ is a semilattice of ``descriptions'' with similarity operation $\sqcap$ inducing natural partial order $(D,\sqsubseteq)$ given by $x \sqsubseteq y \Leftrightarrow x \sqcap y =x$ and $\delta:T\rightarrow D$ is a mapping from transactions to their descriptions.
Then, a binary dataset is $(T,(2^\A,\cap),\delta)$, where $\sqcap$ is $\cap$, $\sqsubseteq$ is $\subseteq$, and $\delta(t)=\{i \in \A \mid (t,i) \in R\}$. We use the pattern structure representation in order to iteratively modify the pattern space which is discussed later.
Any subset of $T$ is called a tidset and any subset of $\A$ is called an \textit{itemset}.
An example of a dataset is given in Figure~\ref{tbl:stability-context}.

The following mappings give a Galois connection between the powerset of transactions and the semilattice of descriptions, e.g. $(2^\A,\cap)$.
\begin{align*}
  d(A) &:= \underset{t \in A}{\bigsqcap}\delta(t), &\text{for } A \subseteq T\\
  t(x) &:= \{t \in T \mid x \sqsubseteq \delta(t)\}, &\text{for } x \in D
\end{align*}

In case of a binary dataset, the mapping $d(A)$ returns the maximal itemset common to all transactions in $A$, while the mapping $t(x)$ returns the set of all transactions whose descriptions are supersets of $x$.
One can define closure operators and the corresponding closed tidsets and closed descriptions: $c_t := t \circ d$ and $c_d := d \circ t$ are closure operators, while the closed tidset $A$ and closed itemset $x$ are given by $c_t(A)=A$ and $c_d(x)=x$. As stated in~\cite{Pasquier1999}, this type of closure (based on Galois connection and $\sqcap$ operation) is equivalent for itemsets to the closure wrt. ``counting inference", which is common in data mining. However, the former definition unifies very important notions of ``maximal common part", closure, and lattices of closed patterns, so we shall keep to it in this paper.

A \textit{concept} of a dataset $(T,(D,\sqcap),\delta)$ is a pair $(A,x)$, where $A \subseteq T$, called \textit{extent} and $x \subset \A$, called \textit{intent}, such that $d(A) = x$ and $t(x) = A$. In this case both $A$ and $x$ are closed tidset and itemset, respectively.
The set of concepts is partially ordered w.r.t. inclusion on extents, i.e., $(A_1,x_1)\leq (A_2,x_2)$ iff $A_1\subseteq A_2$ (or, equivalently, $x_2\sqsubseteq x_1$), forming a lattice. An example of a lattice corresponding to the binary dataset in Figure~\ref{tbl:stability-context} is given in Figure~\ref{fig:stability}.

In the reminder we need some results from pattern structures for justifying our approach. Moreover, our approach is also applicable to more complex data given by general datasets (pattern structures). For example, \AlgoName was successfully applied to interval-tuple datasets \cite{BuzmakovPKDD2015}.

%
\subsection{Projections of Datasets}

The approach proposed in this paper is based on projections introduced for reducing complexity of computing with pattern structures~\cite{Ganter2001}.

A \textit{projection} $\psi: D \rightarrow D$ is an ``interior operator'', i.e., it is (1)~monotone ($x \sqsubseteq y \Rightarrow \psi(x) \sqsubseteq \psi(y)$), (2)~contractive ($\psi(x) \sqsubseteq x$) and (3)~idempotent ($\psi(\psi(x))=\psi(x)$).
A \emph{projected dataset} $\psi(\PS)=\psi((T,(D,\sqcap),\delta))$ is a dataset $(T,(D_\psi,\sqcap_\psi),\psi \circ \delta)$, where $\psi(D)=\{x \in D \mid \exists x^* \in D : \psi(x^*) = x\}$ is the fixed set of $\psi$ and $\forall x,y \in D, x \sqcap_\psi y := \psi( x \sqcap y )$.

In the case of binary datasets projections correspond to removal of some items, with the respective change of the dataset $(T,(2^\A,\cap),\delta)$. The projection of an itemset $X \subseteq \A$ corresponding to removal of a set of items $Y \subseteq \A$ is given by
\begin{equation}\label{eq:psi-binary}
  \psi(X)=X \cap (\A \setminus Y)=X \setminus Y.
\end{equation}Given a projection $\psi$ we call $\psi(D)=\{x \in D \mid \psi(x)=x\}$ the \emph{fixed set of} $\psi$. The fixed set contains those itemsets that contain no items from the set $Y$ (the set of removed items). The projections are ordered w.r.t. inclusion of the fixed points (or by inclusion of the sets of removed items in the case of binary data), i.e., $\psi_1 < \psi_2$, if $\psi_1(D) \subseteq \psi_2(D)$, we say that $\psi_1$ is simpler than $\psi_2$ or that $\psi_2$ is more detailed than $\psi_1$.


Our algorithm is based on this order on projections. The simpler a projection $\psi$ is, the less itemsets we can find in $\psi(\PS)$, and the less computational efforts one should take. Thus, we compute a set of itemsets for a simpler projection, then we remove unpromising itemsets, extend our dataset and the found itemsets with more items (to a more detailed projection). This allows us to reduce the size of the pattern space with a simpler projection and lower computational complexity.

\section{\AlgoName{} Algorithm}\label{sect:algo}
\subsection{Anti-monotonicity w.r.t. a Projection}

Our algorithm is based on the projection-antimonotonicity. Many interestingness measures for itemsets, e.g., stability~\cite{KuznetsovStability2007}, robustness of closed itemsets~\cite{Tatti2014}, or cosine interest~\cite{Cao2014}, are not (anti-)monotonic w.r.t. inclusion order on itemsets. A measure $\M$ is called \emph{anti-monotonic} if for two itemsets $x \sqsubseteq y$, $\M(x) \geq \M(y)$. For instance, support is an anti-monotonic measure w.r.t. itemset inclusion order and it allows for efficient generation of itemsets with support larger than a threshold~\cite{Agrawal1994,Mannila1994,Pasquier1999}. 
The projection-antimonotonicity is a generalization of standard anti-monotonicity and allows for efficient processing a larger set of interestingness measures.

Given a projection $\psi$ corresponding to the removal of items $Y$, \textit{preimages} of an itemset $X$ (we assume $X \cap Y = \emptyset$) for $\psi$ is the set of itemsets $\{Z\}$ such that $\psi(Z)=X$. It can be seen that the set of preimages is given by $\mathtt{Preimages(Y)}=\set{Z \subseteq \A \mid X \subseteq Z \subseteq X \cup Y}$. In particular $X$ is also a preimage of itself.

An \textit{anti-monotonic measure $\M$ w.r.t. projection $\psi$} (or just a \textit{projection-antimonotonic measure}) is a measure which does not increase its value on any premiage of any itemset $X$ for $\psi$. Since any preimage of $X$ is a superset of $X$, then any anti-monotonic measure is also a projection-antimonotonic measure.

\begin{example}
  Let us consider the dataset in Figure~\ref{tbl:stability-context}. If $\M$ is an interestingness measure w.r.t. a projection $\psi$ and $\psi$ removes item $i_5$, then $\M(\set{i_3}) \geq \M(\set{i_3,i_5})$. However it is \textbf{not} necessary that $\M(\set{i_3})\geq \M(\set{i_3,i_4})$.
\end{example}

Thus, given a measure $\M$ anti-monotonic w.r.t. a projection $\psi$, if $y$ is an itemset such that $\M_\psi(y) < \theta$, then $\M(x)<\theta$ for any preimage $x$ of $y$ for $\psi$. Hence, if, given an itemset $y$ of $\psi(\PS)$, one can find all itemsets $x$ of $\PS$ such that $\psi(x)=y$, it is possible to find the itemsets in $\psi(\PS)$ and then to prune them w.r.t. $\M_\psi$, and finally to compute the preimages of the pruned set of itemsets only. 
It allows one to earlier cut unpromising branches of the search space or adjust a threshold for finding only a limited number of best itemsets.
%
%

However, given just one projection, it can be hard to efficiently discover the best itemsets, since the projection is either hard to compute or the number of unpromising itemsets that can be pruned is not high. Corespondingly we need \emph{a chain of projections} $\psi_0 < \psi_1 < \cdots < \psi_k=\mathbb{1}$, where concepts for $\psi_0(\PS)$ can be easily computed and $\mathbb{1}$ is the identity projection, i.e., $(\forall x)\mathbb{1}(x)=x$. For example, to find frequent itemsets, we typically search for small frequent itemsets and then extend them to larger ones. It corresponds to the extension to a more detailed projection. In particular for binary dataset a chain of projections can be instantiated as a consequent update of a binary dataset with new items.

Chain of projections is a generalization of accessible system~\cite{Boley2010}. Given a set of items $\A$ and a subset of its powerset $\mathcal{F}\subseteq 2^\A$, the system $(\A,\mathcal{F})$ is accessible if $\forall X \in \mathcal{F} \setminus \set{\emptyset}$ there is $i \in \A$ such that $X \setminus \set{i} \in \mathcal{F}$. Any constraint (or measure) on $2^\A$ produces a system of sets. If this system is accessible, then the measure is locally anti-monotonic. 
\begin{proposition}
  A chain of projections can be represented as a sequence of systems
  $(\A_i,\mathcal{F}_i)$ such that $\A_i \subset \A_{i+1}$ and any
  element $x \in \mathcal{F}_{i+1}$ is either (1) $x \in \mathcal{F}_i$, or
  (2) $\exists e \in \A_{i+1}\setminus\A_i$ such that
  $(x \setminus \set{e}) \in \mathcal{F}_i$, (3) or $x$ accesible in
  $\mathcal{F}_{i+1}$.
\end{proposition}
\begin{proof}
  (1) by idempotency of projections, (2) by contractivity, (3) for deletion of several items.
\end{proof}



%
\subsection{Algorithms}

\SetAlFnt{\footnotesize}
\begin{algorithm}[t]
  \AlgoDisplayBlockMarkers
  \SetAlgoBlockMarkers{}{}
  \SetAlgoNoEnd
  \SetKwProg{Fn}{Function}{}{}
  \SetKwFunction{ExtProj}{ExtendProjection}
  \SetKwFunction{Preimages}{Preimages}
  \SetKwFunction{ThetaAlgo}{Algorithm\_\ThetaAlgoName{}}
  \SetKwFunction{FindPatterns}{FindPatterns}

  \KwData{
    A dataset $\PS$, a chain of projections $\Psi=\set{\psi_0,\psi_1,\cdots,\psi_k}$, an anti-monotonic measure $\M$ for the chain $\Psi$, and a threshold $\theta$ for $\M$.
  }
  \Fn{\ExtProj{$i$, $\theta$, $\mathcal{P}_{i-1}$}}{
    \KwData{
      $i$ is the projection number to which we should extend ($0<i\leq k$), $\theta$ is a threshold value for $\M$, and $\mathcal{P}_{i-1}$ is the set of itemsets for the projection $\psi_{i-1}$.
    }
    \KwResult{
      The set $\mathcal{P}_i$ of all itemsets with the value of measure $\M$ higher than the threshold $\theta$ for $\psi_i$.
    }
    $\mathcal{P}_{i}\longleftarrow\emptyset$\;
    \ForEach{$p \in \mathcal{P}_{i-1}$}{
      $\mathcal{P}_i\longleftarrow\mathcal{P}_i \cup \Preimages{i,p}$	
    }
    \ForEach{$p \in \mathcal{P}_i$}{
      \If{$\M_{\psi_i}(p) \leq \theta$}{
        $\mathcal{P}_i \longleftarrow \mathcal{P}_i \setminus \set{p}$
      }
    }
  }
  \Fn{\ThetaAlgo}{
    \KwResult{
      The set $\mathcal{P}$ of all itemsets with a value of $\M$ higher than the threshold $\theta$ for $\PS$.
    }
    $\mathcal{P}\longleftarrow\FindPatterns{$\theta,\psi_0$}$\;
    \ForEach{$0 < i \leq k$}{
      $\mathcal{P}\longleftarrow\ExtProj{$i, \theta, \mathcal{P}$}$\;
    }
  }
  \caption{
    \ThetaAlgoName{}
    \label{alg:theta-sofia}
    \vspace*{-2mm}
  }
\end{algorithm}

\begin{algorithm}[th]
  \AlgoDisplayBlockMarkers
  \SetAlgoBlockMarkers{}{}
  \SetAlgoNoEnd
  \SetKwProg{Fn}{Function}{}{}
  \SetKwFunction{ExtProj}{ExtendProjection}
  \SetKwFunction{Preimages}{Preimages}
  \SetKwFunction{Algo}{Algorithm\_\AlgoName{}}
  \SetKwFunction{FindPatterns}{FindPatterns}
  \SetKwFunction{AdjTheta}{AdjustTheta}
  \SetKwFunction{FltrPatterns}{PrunePatterns}
  
  \KwData{
    A dataset $\PS$, a chain of projections $\Psi=\set{\psi_0,\psi_1,\cdots,\psi_k}$, an anti-monotonic measure $\M$ for the chain $\Psi$, and a threshold $L$ for the maximal number of preserved itemsets.
  }
  \Fn{\Algo}{
    \KwResult{
      The threshold $\theta$ ensuring that the cardinality of the set $\mathcal{P}$ is bounded by $L$ in any step of the algorithm.
      The set $\mathcal{P}$ of all itemsets with the value of measure $\M$ higher than the threshold $\theta$.
    }
    
    $\theta\longleftarrow\theta_{\min{}}$
    $\mathcal{P}\longleftarrow\FindPatterns{$\psi_0$}$\;
    \ForEach{$0 < i \leq k$}{
      $\theta \longleftarrow \AdjTheta{$\theta,L,\mathcal{P}$}$\;
      $\mathcal{P} \longleftarrow \FltrPatterns{$\theta,\mathcal{P}$}$\;	
      $\mathcal{P}\longleftarrow\ExtProj{$i,\theta,\mathcal{P}$}$\;
    }
  }
  \caption{
    \AlgoName{} for finding itemsets in $\PS$ with the bounded cardinality of the set $\mathcal{P}$. \label{alg:sofia}
    \vspace*{-7mm}
  }
\end{algorithm}

Given a dataset $\PS$ and a measure anti-monotonic w.r.t. a chain of projections, if we are able to find all preimages of any element in the fixed set of a projection $\psi_i$ that belong to a fixed set of the next projection $\psi_{i+1}$, then we can find all itemsets of the dataset $\PS$ with a value of $\M$ higher than a given threshold $\theta$. We call the respective algorithm \ThetaAlgoName{} (Algorithm~\ref{alg:theta-sofia}). In line 9 we find all itemsets for the dataset $\psi_0(\PS)$ satisfying the constraint w.r.t. the measure $\M$. Then in lines 10-11 we iteratively extend projections from simpler to more detailed ones. The extension is done by constructing the set $\mathcal{P}_i$ of preimages of the set $\mathcal{P}_{i-1}$ (lines 2-4) and then by removing the itemsets that do not satisfy the constraint from $\mathcal{P}_i$ (lines 5-7).
This listing provides a sketch of the algorithm omitting possible engineering improvements for the sake of simplicity. Most of the known improvements are applicable here, i.e., the ones from \cite{Uno2005}. In particular, a canonical order on itemsets is used for mining closed pattern (the theoretical basis for such mining is given in next subsections).

The algorithm is sound and complete, since first, an itemset $p$ is included into the set of preimages of $p$ (since $\psi(p)=p$) and second, if $\M(p) < \theta$, then we remove the itemset $p$ from the set $\mathcal{P}$ and the measure value of any preimage of $p$ is less than $\theta$ by the projection chain anti-monotonicity of $\M$.
The worst case time complexity for the general case of patterns of \ThetaAlgoName{} algorithm is
\begin{align}\label{eq:complexity-theta}
  \mathbb{T}(\text{\ThetaAlgoName})&=\mathbb{T}(FindPatterns(\psi_0))+\notag\\
                                   &\hspace{-12mm}+k\cdot\underset{0<i\leq k}{\max}|\mathcal{P}_i|\cdot(\mathbb{T}(Preimages)+\mathbb{T}(\M)),
\end{align}
where $k$ is the number of projections in the chain, $\mathbb{T}(\mathcal{X})$ is the time for computing the operation $\mathcal{X}$.
Since projection $\psi_0$ can be chosen to be very simple, in a typical case the complexity of $FindPatterns(\theta,\psi_0)$ can be low or even constant. The complexities of $Preimages$ and $\M$ depend on the measure, the chain of projections, and the kind of patterns.
In many cases $\underset{0<i\leq k}{\max}|\mathcal{P}_i|$ can be exponential in the size of the input, because the number of patterns can be exponential. It can be a difficult task to define the threshold $\theta$ such that the maximal cardinality of $\mathcal{P}_i$ is not larger than a given number. 
Thus, we introduce \AlgoName{} algorithm (Algorithm~\ref{alg:sofia}), which automatically adjusts  threshold $\theta$ ensuring that $\underset{0<i\leq k}{\max}|\mathcal{P}_i| < L$. Here $L$ can be considered as a constraint on the memory used by the algorithm.
The only difference of \AlgoName{} w.r.t. \ThetaAlgoName{} is that after performing an operation that changes the set $\mathcal{P}$ (lines 2 and 6 in Algorithm~\ref{alg:sofia}) it adjusts $\theta$ in such a way that the cardinality of $\mathcal{P}$ does not exceed the parameter $L$. 
It can be seen from~(\ref{eq:complexity-theta}) that \AlgoName{} has polynomial time complexity if $\M$ and $Preimages$ are polynomial. Indeed, according to (\ref{eq:psi-binary}) if a projection removes only one item, the cardinality of $\mathtt{Preimages}$ is always 2.
Thus, the worst case complexity for \ThetaAlgoName{} is
\begin{align}\label{eq:complexity-theta-bin}
  \mathbb{T}(\text{\ThetaAlgoName}_{\text{binary}})=|\A|\cdot\underset{0<i\leq N}{\max}|\mathcal{P}_i|\cdot\mathbb{T}(\M).
\end{align}
We notice that every $\mathcal{P}_i$ is a solution for the projected dataset. Thus, this algorithm has  incremental polynomial delay.
However, if we fix the available memory $L$, the complexity of \AlgoName{} for binary data is $|\A|\cdot L\cdot\mathbb{T}(\M)$, i.e., it becomes input polynomial modulo complexity of the measure.

To wrap up, in this subsection we have introduced an algorithm for finding top-K itemsets in polynomial time. It is important to notice that the found set of itemsets is exactly the best itemsets w.r.t. to \textDelta-measure and should not be mixed up with an approximation.

\subsubsection*{Efficiency Considerations}

Recently much work have been done in finding good strategies of enumerating (closed) patterns. Most of them start from the smallest patterns and then iteratively generate larger patterns. It can be naturally expressed as a chain of functions $\psi_i$ that are contractive ($\psi_i(X) \sqsubseteq X$) and idempotent ($\psi_i(\psi_i(X))=\psi_i(X)$). These functions can be ordered by inclusion of fixed sets because of idempotency. Since these functions are contractive, only  patterns larger than a pattern $X$ are preimages of $X$. Thus, most of the approaches for itemset mining can be formalized by means of a chain of such functions.
However, in this work we require a chain of projections, i.e., functions $\psi_i$, to be also monotonic. It allows us to efficiently mine robust and stable patterns discussed in Section~\ref{itemset-constraints}.
This additional monotonicity still allows one to formalize developed approaches for itemset mining as a chain of projections. However, in this work we does not discuss this formalization and focus on the efficient mining of patterns for nonmonotonic constraints.

\subsection{\AlgoName{} Algorithm for Closed Itemsets}\label{sect:closed-patterns}
Closed frequent itemsets are widely used as a condensed representation of all frequent itemsets since~\cite{Pasquier1999}. Here we show how one can adapt our algorithm for closed patterns.
A closed pattern in $\psi_{i-1}(\PS)$ is not necessarily closed in $\psi_i(\PS)$. Indeed, if we take the example in Figure~\ref{fig:stab-example}, the pattern $\set{i_1}$ is closed in $(T,\set{i_1,i_2},R_2)$ but no more closed in $(T,\set{i_1,i_2,i_3},R_3)$. However, the extents (closed tidsets) of $\psi(\PS)$ are extents of $\PS$~\cite{Ganter2001}. Thus, we associate the closed patterns with extents, and then work with extents instead of patterns, i.e., a dataset $\PS=(T,(2^\A,\cap),\delta)$ is transformed into $\PS_C=(T,(D_C,\sqcap_C),\delta_C)$, where $D_C=2^T$. Moreover, for all $a,b \in D_C$ we have $a \sqcap_C b = t(d(a) \sqcap d(b))$, where $t$ and $d$ operators are computed in $\PS$ and $\delta_C(t \in T)=\set{t}$. Hence, every pattern $p$ in $D_C$ corresponds to a closed pattern $d(p)$ in $2^\A$.
A projection $\psi$ of $\PS$ induces a projection $\psi_C$ of $\PS_C$, given by $\psi_C(A \subseteq T)=t(\psi(d(A)))$ with $t$ and $d$ computed for $\PS$.

In the next section we discuss some measures that are anti-monotonic w.r.t. a projection (rather than just anti-monotonic). In the end of the next section we provide an example of how \AlgoName{} works.

\section{Itemset Constraints}\label{itemset-constraints}

\begin{table}[ht]
  \vspace{-5mm}
  \centering
  \caption{Values of different measures for closed itemesets of context in Figure~\ref{tbl:stability-context}.}
  \label{tbl:measures-values}
  \scalebox{0.8}{
  \begin{tabular}{l|cccc}
    \hline \hline
    Itemset $X$ 
    & Cosine & $\stab(X)$ & $\underset{\alpha=0.9}{\rbst}(X)$ & $\Delta(X)$ \\
    \hline
    $\emptyset$
    & $+\infty$&  0.47   & 0.89991 & 1 \\
    $\bf\set{i_3}$
    &\bf 1   &\bf 0.69   &\bf 0.9963 & 3 \\
    $\set{i_1,i_3}$
    &  0.5   &   0.5     &  0.9 & 1 \\
    $\set{i_2,i_3}$
    &  0.5   &   0.5     &  0.9 & 1 \\
    $\set{i_3,i_4}$
    &  0.5   &   0.5     &  0.9 & 1 \\
    $\set{i_3,i_5}$
    &  0.5   &   0.5     &  0.9 & 1 \\
    $\set{i_6}$
    &    1   &   0.5     &  0.9 & 1 \\
    \hline \hline
  \end{tabular}
  }
  \vspace{-5mm}
\end{table}

\subsection{Cosine Interest of an Itemset}

The first projection-antimonotonic measure we consider is cosine interest~\cite{Cao2014}. It is defined by
\begin{equation}
  \label{eq:cosine-interest}
  \mathtt{Cosine}(X)=\frac{|t(X)|}{\sqrt[|X|]{\prod_{i \in X}|t(\set{i})|}},
\end{equation}i.e., a cosine interest of $X$ is the support of $X$ over the geometric mean of supports of single items from $X$. As the authors of~\cite{Cao2014} have shown this measure is not (anti-)monotonic. Then, they also have shown that if we traverse the search space from less supported items to more supported items the cosine interest never decreases. Indeed, given an itemset $X$ and an item $i$ such that $i \not \in X$ and $(\forall j \in X) |t(\set{i})|\geq|t(\set{j})|$, we can see that $\mathtt{Cosine}(X) \geq \mathtt{Cosine}(X\cup \set{i})$ since the itemset support cannot increase while the geometric mean cannot decrease in this case.

To work with cosine interest we can define a projection chain that adds items from less supported ones to more supported, i.e., $\psi_1$ corresponds to removal of all but the least frequent item from the dataset, $\psi_2$ corresponds  to removal of all but two least frequent items and so on. Then, cosine interesting itemsets can be mined by \AlgoName{}.
However this measure is locally anti-monotonic,
in the next subsection we consider two proper nonmonotonic measures.

\subsection{Stability and Robustness of an Itemset}

Stability~\cite{KuznetsovStability2007} and robustness~\cite{Tatti2014} are similar measures when applied to closed itemsets. They measure
independence of an itemset w.r.t. subsampling. Stability can only be applied to closed itemsets, while robustness is defined for any type of itemset constraints (closed itemsets, generators, \textit{etc}.). However, in case of closed itemsets neither of them is (anti-)monotonic. Indeed, when robustness is based on an anti-monotonic constraint, it is anti-monotonic. However, closedness of itemsets is not an anti-monotonic constraint. Since stability and robustness are similar, we define them on a similar basis.

Given a dataset $\K=(T,\A,R)$, a triple $(S,\A,R)$ where $S \subseteq T$ is called a \emph{subdataset} of $\K$. If we give a weight to every subdataset of $\K$, then we can find the sum of weights of all subdatasets of $\K$ where an itemset $X$ is closed. This sum gives us stability or robustness of the closed itemset $X$ depending on how we define the weights of subdatasets.

In the case of \emph{stability} the weights $w$ of  all subdatasets $\K_s$ of $\K$ are equal, i.e., $w(\K_s)=2^{-|T|}$. In this case we consider every subdataset equally probable and compute the probability that the itemset $X$ is closed.

\begin{example}
  Consider example in Figure~\ref{tbl:stability-context}. The set of concepts (the pattern of every concept is a closed itemset) is shown in Figure~\ref{fig:stability}. Stability of every closed itemset is shown in Table~\ref{tbl:measures-values}.
  Let us consider the highlighted itemset $X=\set{i_3}$. There are $2^5$ possible subdatasets. Only in the following 10 subdatasets $X$ is not closed (only the set of transactions for every subdataset is given): $\emptyset, \set{t_1},\dots,\set{t_5},\set{t_1,t_5},\set{t_2,t_5},\set{t_3,t_5},\set{t_4,t_5}$. Thus, stability of $X$ can be found as $\stab(X)=1-10\cdot 2^{-5}=0.69$.

  It should be noticed that stability of all comparable itemsets in the lattice is smaller than stability of $X$, which highlights the nonmonotonicity of stability.
\end{example}

In the case of \emph{robustness} the weights $w$ of subdatasets are computed differently. These weights depend on a parameter $0 \leq \alpha \leq 1$ denoting the probability of a transaction to be retained in the dataset. The weight of a subdataset $\K_s=(S,\A,R)$ of $\K=(T,\A,R)$ corresponds to the probability of obtaining $\K_s$ by removing every single transaction from $\K$ with probability $1-\alpha$: $w(\K_s)=\alpha^{|S|}\cdot(1-\alpha)^{|T|-|S|}$.

\begin{example}
  Consider example in Figure~\ref{tbl:stability-context}. Robustness for $\alpha=0.9$ for every closed itemset is shown in Table~\ref{tbl:measures-values}.
  Let us consider the highlighted itemset $X=\set{i_3}$. It is not closed in the same as above 10 subdatasets but their weights are different (the weights are shown in superscripts):
  {\scriptsize
    $\emptyset^{w=10^{-5}}, \set{t_1}^{w=9\cdot 10^{-5}}, \dots,\set{t_5}^{w=9\cdot 10^{-5}}, \set{t_1,t_5}^{w=8.1\cdot 10^{-4}}$, $\set{t_2,t_5}^{w=8.1\cdot 10^{-4}},\set{t_3,t_5}^{w=8.1\cdot 10^{-4}},\set{t_4,t_5}^{w=8.1\cdot 10^{-4}}$.}
  Thus, robustness of $X$ for $\alpha=0.9$ is equal to $\rbst_{\alpha=0.9}(X)=0.9963$. It can be verified that robustness is not an anti-monotonic measure.
\end{example}

It is not hard to show that independently of the weights $w$ of subdatasets, stability and robustness are anti-monotonic measures w.r.t. any projection.

\begin{proposition}\label{prop:rbst-stab-monotonicity}
  Stability and robustness are anti-monotonic measures w.r.t. any projection.
\end{proposition}
\begin{proof}
  Here we want to show that for any projection $\psi$ if a pattern $X$ is closed in a subdataset $\K_s$ then $\psi(X)$ is closed in $\psi(\K_s)$, where $\K_s=(S,\A,R)$ is a subdataset of $\K=(T,\A,R)$ with $S \subseteq T$. We note that if $X$ is closed in $\K_s$ it is also closed in $\K$. And since $\psi(X)$ closed in $\psi(\K)$, then for projection $\psi_C$ from Section~\ref{sect:closed-patterns} we have $d(\psi_C(t(X)))=\psi(X)$. Hence, we can work with images of $\psi$ on closed patterns in order to find the corresponding images of $\psi_C$.

  Let $Y=d(t(\psi(X)) \cap S)$ be a closure of $\psi(X)$ in $\K_s$. Since $\psi(X) \sqsubseteq X$, then $t(\psi(X)) \supseteq t(X)$. Hence $S \cap t(\psi(X)) \supseteq S \cap t(X)$. Then $Y=d(S \cap t(\psi(X))) \sqsubseteq d(S \cap t(X))=X$, since $Y$ is the closure of $\psi(X)$ in $\psi(\K_s)$ and $X$ is the closure of $X$ in $\K_s$. Thus, we have $Y \sqsubseteq X$ and $c_t(Y) \sqsubseteq c_t(X)=X$. Because of monotonicity of projections one has $\psi(c_t(Y)) \sqsubseteq \psi(X)$ and hence $Y \sqsubseteq \psi(X)$.

  Since $Y$ is the closure of $\psi(X)$ in $\K_s$, then $Y \sqsupseteq \psi(X)$. Hence $Y=\psi(X)$.
\end{proof}

\subsubsection*{Estimates of Stability and Robustness}

For both stability and robustness it is shown that the corresponding constraint is NP-hard~\cite{KuznetsovStability2007,Tatti2014}. Thus, for efficient mining, estimates of stability and robustness are essential. Here we introduce a fast computable estimate of robustness in the same way we did it for stability in~\cite{BuzmakovICFCA2014}.

Let us consider closed itemsets $X$ and $Y$ such that $X \subset Y$. \textit{Can we define the subdatasets where $X$ is not  closed?}
Let us define $\Delta(X,Y)$ as the cardinality of the set of transactions described by $X$ but not by $Y$: $\Delta(X,Y)=t(X) \setminus t(Y)$. This set is not empty since $X \neq Y$ and they are closed. It is clear that $X$ is not closed in any subdataset that removes all transactions from $\Delta(X,Y)$, since $Y$ is a larger itemset with the same support. Then, $\stab(X) \leq 1 - 2^{-\Delta(X,Y)}$ and $\rbst(X) \leq 1 - (1-\alpha)^{\Delta(X,Y)}$ for any closed itemset $Y \supset X$. In particular, we can put $Y$ to the closest closed superitemset of $X$.

In the same way we can take all immediate closed superitemsets of $X$ and take into account all the subdatasets where $X$ is not closed. Since some of the subdatasets are probably counted several times we get the lower bound, i.e., $\stab(X) \geq 1 - \underset{Y\prec X}{\sum}2^{-\Delta(X,Y)}$ and $\rbst(X) \geq 1 - \underset{Y\prec X}{\sum}(1-\alpha)^{\Delta(X,Y)}$.

\begin{proposition}
  Stability and robustness are bounded as follows, where $X\prec Y$ means that $X$ is an immediate closed subitemset of $Y$:

  \vspace*{-3mm}
  {\scriptsize
    \begin{align}
      \label{eq:1}
      1 - \underset{Y\prec X}{\sum}2^{-\Delta(X,Y)} &\leq \stab(X) \leq &1 - 2^{-\Delta(X,Y)} \\
      1 - \underset{Y\prec X}{\sum}(1-\alpha)^{\Delta(X,Y)} &\leq \rbst(X) \leq &1 - (1-\alpha)^{\Delta(X,Y)}
    \end{align}
  }
\end{proposition}

In particular we can see that when $\alpha = 0.5$ the estimates are exactly the same. As it is recently shown~\cite{BuzmakovICFCA2014}, the estimate of stability is quite precise for the concepts with stability close to 1. Then, when $\alpha > 0.5$ the precision of the estimate of robustness is even more precise.

These estimates can be computed in polynomial time in contrast to stability and robustness. And thus we can use one of the bounds as a proxy to stability and robustness. It can be seen that the rankings based on the upper bound of stability and robustness are exactly the same as the ranking based on $\Delta(X)=\underset{Y\prec X}{\min}\Delta(X,Y)$. Although for the lower bound of stability and robustness it is hard to show the projection anti-monotonicity, we can show it for the upper bound. In the following $\Delta(X)$ is called \textDelta-measure.

\begin{proposition}
  \textDelta-measure is an anti-monotonic measure w.r.t. any projection.
\end{proposition}
\begin{proof}
  We remind that for dealing with closed patterns the tidsets are considered as patterns as discussed in Section~\ref{sect:closed-patterns}.
  By properties of projections, if an extent (the tidset of a concept) is found in $\psi(\PS)$, it is necessarily found in $\PS$~\cite{Ganter2001}. Let us consider a tidset $E$ of a concept and a tidset of its descendant $E_c$ in $\psi(\PS)$, where a descendant concept is a concept with a smaller tidset and larger itemset. 
Let us suppose that $E_p$ is a preimage of $E$ for the projection $\psi$. Since $E_c$ and $E_p$ are extents in $\PS$, the set $E_{cp}=E_c \cap E_p$ is an extent in $\PS$ (the intersection of two closed sets is a closed set). Since $E_p$ is a preimage of $E$, then $E_p \not\subseteq E_c$ (otherwise, $E_p$ is a preimage of $E_c$ and not of $E$). Then, $E_{cp} \neq E_p$ and $E_{cp} \subseteq E_p$. Hence, $\Delta(E_{p})\leq |E_{p} \setminus E_{cp}| \leq |E \setminus E_c|$. So, given a preimage $E_p$ of $E$, $(\forall E_c < E)\Delta(E_{p}) \leq |E \setminus E_c|$, i.e., $\Delta(E_p) \leq \Delta(E)$. Thus, we can use \textDelta-measure in combination with \AlgoName{} algorithm.
\end{proof}

\begin{example}
  Consider example in Figure~\ref{tbl:stability-context}. \textDelta-measure for every closed itemset is shown in Table~\ref{tbl:measures-values}.
  Let us consider the highlighted itemset $X=\set{i_3}$ with support equal to 4.
  The closest superitemsets of $X$ are $\set{i_1,i_3}$, $\set{i_2,i_3}$, $\set{i_3,i_4}$, and $\set{i_3,i_5}$, all having support equal to one.
  Thus, \textDelta-measure of $X$ is equal to $\Delta(X)=4-1=3$. It can be noticed that \textDelta-measure is not an (anti-)monotonic measure.
\end{example}

\textDelta-measure is related to the work of margin-closeness of an itemset~\cite{Moerchen2011}. In this work, given a set of patterns, e.g., frequent closed patterns, the authors rank them by the minimal distance in their support to the closest superpattern divided by the support of the pattern. In our case, the minimal distance is exactly the \textDelta-measure of the pattern.

\subsection{Example of Stable Itemsets in Binary Data}\label{sect:example}

\newcommand{\g}{\cellcolor[gray]{.7}}
\begin{table}[t]
  \caption{Patterns given by their extent and their stability in the contexts corresponding to a chain of projections.}
  \label{tab:ex-algorithm-concepts}
  \centering
  \resizebox{\columnwidth}{!}{
    \begin{tabular}{l|c|ccccccc}
      \hline\hline
      \multirow{2}{*}{\#} & \multirow{2}{*}{Pattern Ext.} 
      & \multicolumn{7}{c}{\textDelta-measure} \\
                          & & $\A_0$ & $\A_1$ & $\A_2$ & $\A_3$ & $\A_4$ & $\A_5$ & $\A_6$ \\
      \hline
      1 & 12345 
      &  5  &  4   &   4  &\g 1  &\g 1  &\g 1  &\g 1  \\
      2 & 1 
      & --  &\g 1  &\g 1  &\g 1  &\g 1  &\g 1  &\g 1  \\
      3 & 2
      & --  &  --  &\g 1  &\g 1  &\g 1  &\g 1  &\g 1  \\
      4 & 1234
      & --  &  --  &  --  &   3  &   3  &   3  &   3  \\
      6 & 3
      & --  &  --  &  --  &  --  &\g 1  &\g 1  &\g 1  \\
      7 & 4
      & --  &  --  &  --  &  --  &  --  &\g 1  &\g 1  \\
      8 & 5
      & --  &  --  &  --  &  --  &  --  &  --  &\g 1  \\
      \hline\hline
    \end{tabular}
  }
  \vspace{-3mm}
\end{table}

Let us consider the example in Figure~\ref{fig:stab-example} and show how we can find all \textDelta-stable itemsets with threshold $\theta=2$. We have a binary dataset $\K=(T,\set{i_1,\cdots,i_6},R)$. Let us denote $\A_i=\set{i_1,\cdots,i_i}$. The sets $\A_i$ correspond to a chain of projections.

In Table~\ref{tab:ex-algorithm-concepts} all closed itemsets are given by the corresponding tidsets, i.e., by elements of $D_C$. For simplicity we write $1234$ instead of $\set{t_1,t_2,t_3,t_4}$. For every element \textDelta-measure is shown for every $\A_i$. A cell is shown in gray if the itemset is no more considered (the value of \textDelta{} is less than 2).

For example, in the transition from $\A_2$ to $\A_3$ the set $1234$ is discovered with $\Delta(1234)=3$, but $\Delta(12345)=5-4=1$ which is less than $\theta=2$. Thus, itemset $12345$ is discarded and highlighted gray.
The global process is as follows (for the example in Figure~\ref{fig:stab-example}). In the empty binary dataset $(T,\emptyset,R)$ the first itemset $12345$ is considered. Then, in $(T,\set{i_1},R)$ a possible preimage of $12345$ can be either $12345$ or $12345 \cap t(\set{i_1}) = 1$. The set $12345$ is \textDelta-stable ($\Delta(12345)=4$), while $1$ is not \textDelta-stable ($\Delta(1)=1$) and is discarded. Then, the process continues with $(T,\set{i_1,i_2},R)$ and $12345$ is kept while $12345 \cap t(\set{i_2})=2$ is removed for the same reason as $1$. After that, with $(T,\set{i_1,i_2,i_3},R)$ two preimages are still considered, $12345$ and $1234$. This time $\Delta(1234)=3$, while $\Delta(12345)=1$ and the set $12345$ is discarded. The process continues in the same way with $\Delta(1234)=3$ and all other possible elements are discarded.


%
\section{Experiments and Discussion}\label{sect:experiment}
\newcommand{\LCM}{\texttt{LCMv3}\xspace}
\newcommand{\Charm}{\texttt{Charm-L}\xspace}

\subsection{Comparing Computational Efficiency}

In the first experiment we show the computational efficiency of \AlgoName{} coded in C++\footnote{
  The implementation is available at \url{https://github.com/AlekseyBuzmakov/FCAPS}
}. 
We use public available big  datasets from FIMI\footnote{
  \url{http://fimi.ua.ac.be/data/}
}, LUCS~\cite{Coenen2003}, and UCI~\cite{FrankUCI2010} repositories. The experiments are carried out on an ``Intel(R) Core(TM) i7-2600 CPU @ 3.40GHz'' computer with 8Gb of memory under Ubuntu 14.04.

\begin{table}[t]
  \caption{Computational efficiency of \AlgoName algorithm.}
  \label{tbl:efficiency-results}

  \centering
  \resizebox{\columnwidth}{!}{\small
  \begin{tabular}{l|c|cc|cc|c|c}
    \hline\hline
    Dataset 
    & \multirow{2}{*}{Top-K} & \multirow{2}{*}{$\theta_{Supp.}$} & \multirow{2}{*}{$\theta_{\Delta}$} & \multirow{2}{*}{\LCM} & \multirow{2}{*}{\textDelta} & \Charm & \multirow{2}{*}{\AlgoName} \\
    &&&&&& \footnotesize($\sim$ \Charm + \textDelta) & \\
    \hline
    \multicolumn{7}{c}{FIMI} \\
    \hline
    \multirow{2}{*}{chess}
    & 3   & 1145 & 234 & 1.62   & $>100$ & $>100$ & \bf 0.03 \\
    & 928 & 277  & 98  & $>100$ & --- &$>100$ & \bf 0.13 \\
    \hline
    \multirow{2}{*}{connect}
    & 1    & 25466 & 4224 & 0.21 & 128 &  111  & \bf 0.61 \\
    & 1000 & 8822  & 2602 & 1.25 & $>100$ & $>100$& \bf 1.77 \\
    \hline
    \multirow{2}{*}{mushroom}
    &   1 & 6272 & 2256 & $<0.01$ & 0.07 & \bf 0.01 & 0.05\\
    & 722 & 216  & 193  &   0.06  & 2.12 & 0.50 & \bf 0.23\\
    \hline
    \multirow{2}{*}{pumsb}
    & 1 & 33128 & 2035 & 0.15 & $>300$ & 36.7 & \bf 0.8 \\
    & 984 & 8793 & 865 & $>300$ & --- & $>300$ & \bf 38.7 \\
    \hline
    \multirow{2}{*}{pumsb*}
    & 1   & 30787 & 8090 & 0.04 & 1.42 & 0.16 & \bf 0.65 \\
    & 997 & 2808  & 834  & 4.47 & $>300$ & $>300$ & \bf 27.8 \\
    \hline

    \multicolumn{7}{c}{LUCS} \\
    \hline
    \multirow{2}{*}{adult}
    & 1   &34338&6939 & 0.01 & 0.78 & 0.05 & \bf 0.20 \\
    & 998 & 674 & 446 & 0.11 & 16.45 & 2.15 & \bf 1.27 \\
    \hline
    \multirow{2}{*}{waveform}
    & 1   &3424 &1179 & $\bf <0.01$ & \bf 0.01 & \bf 0.01 & 0.03 \\
    & 984 & 401 & 141 & 0.09 & 4.42 & 1.24 & \bf 0.25 \\
    \hline

    \multicolumn{7}{c}{UCI} \\
    \hline
    \multirow{2}{*}{plants}
    & 1   & 11676& 6154  & $<0.01$   & 0.11 & 0.02 & \bf 0.11 \\
    & 984 & 649  & 148  & $>100$ & --- & $>100$ & \bf 0.96 \\

    \hline\hline
  \end{tabular}
  }
\end{table}

We should note  two points here. First, to the best of our knowledge \AlgoName{} is the first algorithm that computes top \textDelta-stable and robust itemsets, so there are no direct competitors. Moreover, computing  \textDelta-measure for an itemset requires either a known partial order of itemsets or a search for its descendants (closed supersets). Thus, as an approximate competitors we decided to use two algorithms \LCM~\cite{Uno2005} and \Charm~\cite{Zaki2005}. The first one is one of the most efficient algorithm for itemset mining that should be followed by \textDelta-measure computation for every concept. \Charm is less efficient than \LCM, but allows one to find the partial order of itemsets necessary for the fast computation of \textDelta-measure.

Second, the current implementation of \AlgoName{} does not use most of the modern optimization techniques, e.g., like in \LCM~\cite{Uno2005}. The current implementation relies only on the so-called conditional database, i.e., where for every tidset $X$ the items that belong to all transactions from $t(X)$ and the items that belong to neither transactions from $t(X)$ are recorded~\cite{Uno2005}. But nevertheless, the computation with the current implementation is efficient.

The experiment is organized as following. First, \AlgoName{} finds around the 1000 most \textDelta-stable itemsets and the maximal support threshold ensuring to find all these the most \textDelta-stable itemsets. Among them we find the most \textDelta-stable itemset (or itemsets if they have the same value of \textDelta-measure) and the corresponding support threshold. 
So \LCM and \Charm are additionally provided with an oracle returning the required support thresholds.
For these two thresholds we run \LCM and \Charm algorithm and register the computation time. In addition for \LCM we register also the time needed for computing \textDelta-measure, while for \Charm this time is insignificant.
In Table~\ref{tbl:efficiency-results} for every dataset we give the results corresponding to every threshold, and the corresponding thresholds for support and \textDelta-measure. For example, for dataset \texttt{chess} we run two experiments. In the first one we search for top-3 \textDelta-stable itemsets having the same value (234) for \textDelta-measure. The less frequent itemsets among these three has support equal to 1145, thus, \LCM and \Charm should be run with this support threshold in order to enumerate all of these itemsets. \LCM finds the corresponding frequent closed itemset in 1.67 seconds, then it takes more than 100 seconds for computing \textDelta-measure. \Charm takes more than 100 seconds and \AlgoName requires only 0.03 seconds. In the second experiment for dataset \texttt{chess} we search for top-928 \textDelta-stable itemsets, all of them have support at least 277 and \textDelta-measure 98.

We boldify the computation time for an algorithm in Table~\ref{tbl:efficiency-results}, if it is better than the time of the competitors. We can see that even \LCM alone does not always beat \AlgoName, while the additional time for \LCM for computing \textDelta-measure is always significant.
There are only two cases when \AlgoName is slightly worse (FIMI-mushroom and LUCS-waveform). For both cases the most stable itemset has a very high support and only a couple of itemsets are frequent enough in both datasets. In contrast, if the frequency of the most \textDelta-stable itemsets is not high, then \AlgoName is many times faster than even \LCM alone.

In these experiments we do not provide the found itemsets since the main focus of our paper is efficiency. However, we highlight that \textDelta-stable patterns are not trivial and can be found deep in the lattice of patterns \cite{Metivier2015}.

\subsection{Scalability}

\begin{table}[t]
  \caption{Scalability of \AlgoName w.r.t. the number of stored patterns}
  \label{tbl:L-scalability}
  \centering
  \scalebox{0.8}{
    \begin{tabular}{l|cccc}
      \hline\hline
      Dataset
      & L=100 & L=1000 & $L=10^4$ & $L=10^5$ \\
      \hline
      \multicolumn{5}{c}{FIMI} \\
      \hline
      chess
      & 0.04 & 0.13 & 1.35 & 14.7 \\
      connect
      & 0.70 & 1.77 & 12.7 & 131 \\
      mushroom
      & 0.1 & 0.29 & 2.62 & 40.5 \\
      pumsb
      & 7.15 & 71.5 &  904   & --- \\
      pumsb*
      & 4.14 & 45.7 &  832   & --- \\
      \hline
      \multicolumn{5}{c}{LUCS} \\
      \hline
      adult
      & 0.30 & 0.99 & 8.79 & 83.97 \\
      waveform
      & 0.06 & 0.18 & 1.97 & 22.13 \\
      \hline
      \multicolumn{5}{c}{UCI} \\
      \hline
      plants 
      & 0.22 & 1.09 & 11.58 & 117.91 \\ 
      \hline\hline
    \end{tabular}
  }
\end{table}

We can study scalability of \AlgoName from different points of view. First, we can measure the time necessary for finding top-$L$ concepts, i.e., how the memory limitation $L$ changes the efficiency. It is shown in Table~\ref{tbl:L-scalability} for the same datasets. We can see that the computation time changes linearly w.r.t. the memory limitation $L$ as it is expected from Eq.~(\ref{eq:complexity-theta-bin}).

\begin{figure}[t]
  \centering
  \includegraphics[width=0.8\columnwidth]{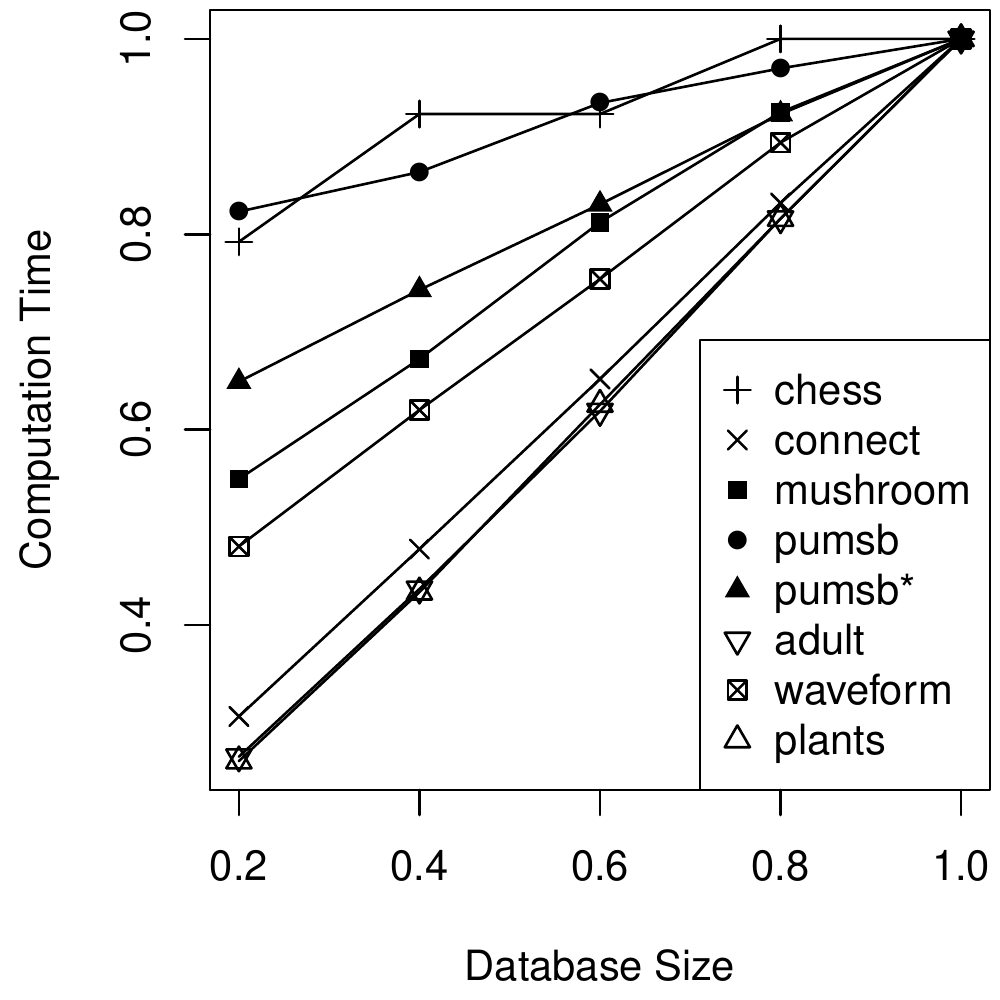}
  \caption{
    Scalability of \AlgoName w.r.t. dataset size. X-axis shows the fraction of objects taken from an original dataset, and Y-axis shows the fraction of time w.r.t. the computational time needed for processing the original dataset.
  }
  \label{fig:scalability-db-size}
  \vspace*{-5mm}
\end{figure}

Finally, we check how computation time depends on the size of the dataset. For that we run our experiments for $L=1000$, and vary the number of transactions in a dataset. We permute several time the order of transactions of the dataset. For every permutation we construct datasets containing certain amount (the size of the dataset) of the first transactions from this permutation. The computation time is averaged over the permutations. Figure~\ref{fig:scalability-db-size} shows the computation time necessary to process a certain fraction of transactions in the dataset. Time is given as a fraction of time for processing the whole dataset. We can see that computation time changes linearly w.r.t. the fraction of processed transactions.

\AlgoName allows limiting the memory in use; thus, as long as the program fits within the memory, which is controllable, the scalability of our approach is linear w.r.t. to the size of the dataset, and consequently can be applied to very huge datasets.

\section{Conclusion}
In this paper we have introduced a new class of interestingness measures, so-called projection-antimonotonic measures. This wide class of measures includes classical anti-monotonic, locally anti-monotonic, and some nonmonotonic measures. We have introduced algorithm \AlgoName, which allows one to efficiently mine patterns w.r.t. projection-antimonotonic measures. We have studied stability and robustness, two projection-antimonotonic measures, and have introduced polynomial estimates of them, called \textDelta-measure. Finally, in the experimental part we have showed that \AlgoName can find \textDelta-stable itemsets much more efficiently than postpruning approaches.

Many directions for future work are promising. First, we should work on adaptation of \AlgoName{} for dealing with different kinds of pattern structures, e.g., based on sequences or graphs~\cite{Kuznetsov2005}. Second, \AlgoName allows one to introduce new  data mining approaches by means of projections of special kind, thus, it is interesting to study possible classes of projections.
Finally, besides robustness and stability we a study of other projection-antimonotonic measures is important.


\PrintBiblio


\newcommand{\etalchar}[1]{$^{#1}$}
\begin{thebibliography}{BHPW10}

\bibitem[AS94]{Agrawal1994}
Rakesh Agrawal and Ramakrishnan Srikant.
\newblock {Fast algorithms for mining association rules}.
\newblock In {\em Proc. 20th int. conf. very large data bases, VLDB}, volume
  1215, pages 487--499, 1994.

\bibitem[BHPW10]{Boley2010}
Mario Boley, Tam{\'{a}}s Horv{\'{a}}th, Axel Poign{\'{e}}, and Stefan Wrobel.
\newblock {Listing closed sets of strongly accessible set systems with
  applications to data mining}.
\newblock {\em Theor. Comput. Sci.}, 411(3):691--700, jan 2010.

\bibitem[BKN14]{BuzmakovICFCA2014}
Aleksey Buzmakov, Sergei~O. Kuznetsov, and Amedeo Napoli.
\newblock {Scalable Estimates of Concept Stability}.
\newblock In Christian Sacarea, Cynthia~Vera Glodeanu, and Mehdi Kaytoue,
  editors, {\em Form. Concept Anal.}, volume 8478 of {\em LNCS}, pages
  161--176. Springer Berlin Heidelberg, 2014.

\bibitem[BKN15]{BuzmakovPKDD2015}
Aleksey Buzmakov, Sergei~O. Kuznetsov, and Amedeo Napoli.
\newblock {Fast Generation of Best Interval Patterns for Nonmonotonic
  Constraints}.
\newblock In Annalisa Appice, Pedro~Pereira Rodrigues, V{\'{i}}tor {Santos
  Costa}, Jo{\~{a}}o Gama, Al{\'{i}}pio Jorge, and Carlos Soares, editors, {\em
  Mach. Learn. Knowl. Discov. Databases}, volume 9285 of {\em LNCS}, pages
  157--172. Springer International Publishing, 2015.

\bibitem[Coe03]{Coenen2003}
F.~Coenen.
\newblock {\em {The LUCS-KDD Discretised and normalised ARM and CARM Data
  Library\footnote{\scriptsize{}\url{http://www.csc.liv.ac.uk/~
  frans/KDD/Software/LUCS_KDD_DN/}}}}.
\newblock Department of Computer Science, The University of Liverpool, UK,
  2003.

\bibitem[CRB08]{Cerf2008}
Lo{\"{i}}c Cerf, C{\'{e}}line Robardet, and Jean-Fran{\c{c}}ois Boulicaut.
\newblock {Data-Peeler: Constraint-Based Closed Pattern Mining in n-ary
  Relations}.
\newblock In {\em SDM'08 Proc. Eighth SIAM Int. Conf. Data Min.}, pages
  37----48. SIAM, 2008.

\bibitem[CWW14]{Cao2014}
Jie Cao, Zhiang Wu, and Junjie Wu.
\newblock {Scaling up cosine interesting pattern discovery: A depth-first
  method}.
\newblock {\em Inf. Sci. (Ny).}, 266(0):31--46, 2014.

\bibitem[FA10]{FrankUCI2010}
A.~Frank and A.~Asuncion.
\newblock {\em {UCI Machine Learning Repository
  [http://archive.ics.uci.edu/ml]}}.
\newblock University of California, Irvine, School of Information and Computer
  Sciences, 2010.

\bibitem[GK01]{Ganter2001}
Bernhard Ganter and Sergei~O. Kuznetsov.
\newblock {Pattern Structures and Their Projections}.
\newblock In Harry~S. Delugach and Gerd Stumme, editors, {\em Concept. Struct.
  Broadening Base}, volume 2120 of {\em LNCS}, pages 129--142. Springer Berlin
  Heidelberg, 2001.

\bibitem[GW99]{Ganter1999}
Bernhard Ganter and Rudolf Wille.
\newblock {\em {Formal Concept Analysis: Mathematical Foundations}}.
\newblock Springer, 1st edition, 1999.

\bibitem[KKN11]{KaytoueIJCAI2011}
Mehdi Kaytoue, Sergei~O. Kuznetsov, and Amedeo Napoli.
\newblock {Revisiting Numerical Pattern Mining with Formal Concept Analysis}.
\newblock In {\em {IJCAI} 2011, Proc. 22nd Int. Jt. Conf. Artif. Intell.
  Barcelona, Catalonia, Spain, July 16-22, 2011}, pages 1342--1347, 2011.

\bibitem[KS05]{Kuznetsov2005}
Sergei~O. Kuznetsov and Mikhail~V. Samokhin.
\newblock {Learning Closed Sets of Labeled Graphs for Chemical Applications}.
\newblock In Stefan Kramer and Bernhard Pfahringer, editors, {\em Inductive
  Log. Program. SE - 12}, volume 3625 of {\em LNCS}, pages 190--208. Springer
  Berlin Heidelberg, lecture no edition, 2005.

\bibitem[Kuz07]{KuznetsovStability2007}
Sergei~O. Kuznetsov.
\newblock {On stability of a formal concept}.
\newblock {\em Ann. Math. Artif. Intell.}, 49(1-4):101--115, 2007.

\bibitem[MLB{\etalchar{+}}15]{Metivier2015}
Jean-Philippe M{\'{e}}tivier, Alban Lepailleur, Aleksey Buzmakov, Guillaume
  Poezevara, Bruno Cr{\'{e}}milleux, Sergei Kuznetsov, J{\'{e}}r{\'{e}}mie {Le
  Goff}, Am{\'{e}}d{\'{e}}o Napoli, Ronan Bureau, and Bertrand Cuissart.
\newblock {Discovering structural alerts for mutagenicity using stable emerging
  molecular patterns}.
\newblock {\em J. Chem. Inf. Model.}, 55(5):925--940, apr 2015.

\bibitem[MTU11]{Moerchen2011}
Fabian Moerchen, Michael Thies, and Alfred Ultsch.
\newblock {Efficient mining of all margin-closed itemsets with applications in
  temporal knowledge discovery and classification by compression}.
\newblock {\em Knowl. Inf. Syst.}, 29(1):55--80, 2011.

\bibitem[MTV94]{Mannila1994}
Heikki Mannila, Hannu Toivonen, and A~Inkeri Verkamo.
\newblock {Efficient Algorithms for Discovering Association Rules}.
\newblock In {\em Knowl. Discov. Data Min.}, pages 181--192, 1994.

\bibitem[PBTL99]{Pasquier1999}
Nicolas Pasquier, Yves Bastide, Rafik Taouil, and Lotfi Lakhal.
\newblock {Efficient Mining of Association Rules Using Closed Itemset
  Lattices}.
\newblock {\em Inf. Syst.}, 24(1):25--46, 1999.

\bibitem[ROK08]{Roth2008}
Camille Roth, Sergei~A. Obiedkov, and Derrick~G. Kourie.
\newblock {On succinct representation of knowledge community taxonomies with
  formal concept analysis}.
\newblock {\em Int. J. Found. Comput. Sci.}, 19(02):383--404, apr 2008.

\bibitem[SC05]{Soulet2005}
Arnaud Soulet and Bruno Cr{\'{e}}milleux.
\newblock {Optimizing constraint-based mining by automatically relaxing
  constraints}.
\newblock In {\em Proc. 5th IEEE Inter- Natl. Conf. Data Min. (ICDM 2005)},
  pages 777--780. IEEE Computer Society, 2005.

\bibitem[SC08]{Soulet2008}
Arnaud Soulet and Bruno Cr{\'{e}}milleux.
\newblock {Adequate condensed representations of patterns}.
\newblock {\em Data Min. Knowl. Discov.}, 17(1):94--110, 2008.

\bibitem[SDB13]{Spyropoulou2013}
Eirini Spyropoulou, Tijl {De Bie}, and Mario Boley.
\newblock {Interesting pattern mining in multi-relational data}.
\newblock {\em Data Min. Knowl. Discov.}, (April):1--42, 2013.

\bibitem[TMC14]{Tatti2014}
Nikolaj Tatti, Fabian Moerchen, and Toon Calders.
\newblock {Finding Robust Itemsets under Subsampling}.
\newblock {\em ACM Trans. Database Syst.}, 39(3):1--27, 2014.

\bibitem[UKA05]{Uno2005}
Takeaki Uno, Masashi Kiyomi, and Hiroki Arimura.
\newblock {LCM Ver.3: Collaboration of Array, Bitmap and Prefix Tree for
  Frequent Itemset Mining}.
\newblock In {\em Proc. 1st Int. Work. Open Source Data Min. Freq. Pattern Min.
  Implementations}, OSDM '05, pages 77--86, New York, NY, USA, 2005. ACM.

\bibitem[VT14]{Vreeken2014}
Jilles Vreeken and Nikolaj Tatti.
\newblock {Interesting Patterns}.
\newblock In Charu~C Aggarwal and Jiawei Han, editors, {\em Freq. Pattern
  Min.}, pages 105--134. Springer International Publishing, 2014.

\bibitem[YCHY08]{Yan2008}
Xifeng Yan, Hong Cheng, Jiawei Han, and Philip~S. Yu.
\newblock {Mining significant graph patterns by leap search}.
\newblock In {\em Proc. 2008 ACM SIGMOD Int. Conf. Manag. data - SIGMOD '08},
  pages 433--444, New York, New York, USA, jun 2008. ACM Press.

\bibitem[ZH05]{Zaki2005}
Mohammed~J. Zaki and Ching-Jui Hsiao.
\newblock {Efficient algorithms for mining closed itemsets and their lattice
  structure}.
\newblock {\em IEEE Trans. Knowl. Data Eng.}, 17(4):462--478, apr 2005.

\end{thebibliography}
\end{document}